\def\year{2019}\relax
\documentclass[letterpaper]{article} 
\usepackage{aaai19}  
\usepackage{times}  
\usepackage{helvet}  
\usepackage{courier}  
\usepackage{url}  
\usepackage{graphicx}  

\usepackage{times} 
\usepackage{amsmath} 
\usepackage{amssymb}  
\usepackage{amsthm}

\usepackage{subfigure}

\usepackage{wrapfig}

\newtheorem{definition}{Definition}

\newtheorem{theorem}{Theorem}

\begin{document}
%
\title{Deictic Image Mapping: An Abstraction For Learning Pose Invariant Manipulation Policies}
\author{Robert Platt, Colin Kohler, Marcus Gualtieri\\ \\
  College of Computer and Information Science, Northeastern University\\
  360 Huntington Ave, Boston, MA 02115, USA\\
  \texttt{\{rplatt,ckohler,mgualti\}@ccs.neu.edu} \\ \\
}
\maketitle
\begin{abstract}
In applications of deep reinforcement learning to robotics, it is often the case that we want to learn pose invariant policies: policies that are invariant to changes in the position and orientation of objects in the world. For example, consider a peg-in-hole insertion task. If the agent learns to insert a peg into one hole, we would like that policy to generalize to holes presented in different poses. Unfortunately, this is a challenge using conventional methods. This paper proposes a novel state and action abstraction that is invariant to pose shifts called \textit{deictic image maps} that can be used with deep reinforcement learning. We provide broad conditions under which optimal abstract policies are optimal for the underlying system. Finally, we show that the method can help solve challenging robotic manipulation problems.
\end{abstract}

\section{Introduction}
\label{Introduction}





Policies learned by deep reinforcement learning agents are generally not invariant to changes in the position and orientation of the camera or objects in the environment. For example, consider the peg-in-hole insertion task shown in Figure~\ref{fig:peg-hole}. A policy that can insert the peg into \textit{Hole A} does not necessarily generalize to \textit{Hole B}. In order to learn a policy that can perform both insertions, the agent must be presented with examples of the hole in both configurations during training. This is a significant problem in robotics because we often want to learn policies most easily expressed relative to an object (i.e. relative to the hole) rather than relative to an image. Without the ability to generalize over pose, deep reinforcement learning must be trained with experiences that span the space of all possible pose variation. For example, in order to learn an insertion policy for any hole configuration, the agent must be trained with task instances for hole configurations spanning $SE(2)$, a three dimensional space. The problem is even worse in $SE(3)$ which spans six dimensions. The need for all this training data slows down learning and hampers the agent's ability to generalize over other dimensions of variation such as task and shape. While pose invariance is not always desirable, the inability to generalize over pose can be a major problem.

\begin{figure}
    \centering
    \includegraphics[width=0.2\textwidth]{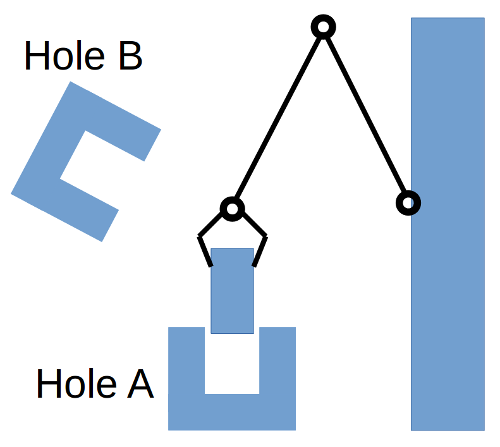}
    \caption{A policy for inserting a peg into Hole A learned using deictic image mapping generalizes to Hole B without any additional training.}
    \label{fig:peg-hole}
\end{figure}


This paper introduces \textit{deictic image mapping}, an approach to encoding robotic manipulation actions that generalizes over the pose of the robot and other objects in the world. There are a couple of key ideas. First, instead of encoding robotic motion at the level of joint velocities or incremental motions, we give the agent access to actions that can move the end-effector directly to a desired position and orientation via a collision free motion planner. Second, we encode these end-to-end motion actions in terms of an image of the world that is centered, aligned, and focused on the target pose of the motion. Since this representation encodes motion actions in terms of the configuration of the world relative to the endpoint of the motion, it generalizes well over changes in object pose and camera viewpoint. We call this a ``deictic'' representation because motion actions are encoded relative to the environment rather than an absolute reference frame. 

One of the interesting aspects of this work is that we show that for a large class of problems, the proposed deictic state and action abstraction induces an MDP homomorphism~\cite{ravindran2004algebraic,ravindran2003smdp}. In consequence, we are assured that optimal policies found in using the deictic representation induce optimal policies for the original problem. An important limitation of our approach is the large action set that is often required: we have as many of 26.9k actions in our experiments. This paper introduces a number of techniques for handling this large branching factor in a DQN framework. Finally, we report on a series of experiments that evaluates the approach both in simulation and in hardware. The results show that the method can solve a variety of challenging manipulation problems with training times less than two hours on a standard desktop GPU system.

\section{Problem Statement}

We will consider problems expressed for a class of robotic systems that we call \textit{move-effect} systems. 


\begin{definition}[Move-Effect System]
A \textit{move-effect} system is a discrete time system comprised of a set of effectors mounted on a fully actuated mobile base that operates in a Euclidean space of dimension $d \in \{2,3\}$. On every time step, the system perceives an image or signed distance function, $I$, and then executes a collision-free motion, $a_m \in \mathcal{A}_M \subset SE(d)$, of the base followed by a motion, $a_e \in \mathcal{A_E}$, of the effectors. 
\end{definition}

A good example of a move-act robot is a robotic arm engaged in prehensile manipulation. The robot may move its hand (i.e. the ``base'') to any desired reachable pose $a_m \in \mathcal{A}_M \subset SE(3)$ via a collision-free motion. Once there, it may execute an effector motion $a_e \in \mathcal{A}_E$ such as opening or closing the hand. More complex manipulation systems can also be formalized this way. For example, a domestic robot that navigates within a house performing simple tasks can be expressed as a move effect system in $SE(2)$ where the simple tasks to be performed are the ``effector motions''.



\subsection{Formulation as a Markov decision process}
\label{sect:mdp_formulation}

A Markov decision process (an MDP) is a tuple $\mathcal{M} = (\mathcal{S},\mathcal{A},T,R)$ where $\mathcal{S}$ denotes a set of system states and $\mathcal{A}$ a discrete set of actions. $T : \mathcal{S} \times \mathcal{A} \times \mathcal{S} \rightarrow [0,1]$ denotes the transition probability function where $T(s_t,a_t,s_{t+1})$ is the probability of transitioning to state $s_{t+1}$ when action $a_t$ is taken from state $s_t$. $R:\mathcal{S} \times \mathcal{A} \rightarrow \mathbb{R}$ denotes the expected reward of executing action $a$ from state $s$. The solution to an MDP is a control policy $\pi : \mathcal{S} \rightarrow \mathcal{A}$ that maximizes the expected sum of future discounted rewards when acting under the policy. We will assume that a motion planner (e.g. an RRT or trajectory optimization planner) is available that can find a collision free path to any reachable pose.


\begin{figure}
    \centering
    \includegraphics[width=0.4\textwidth]{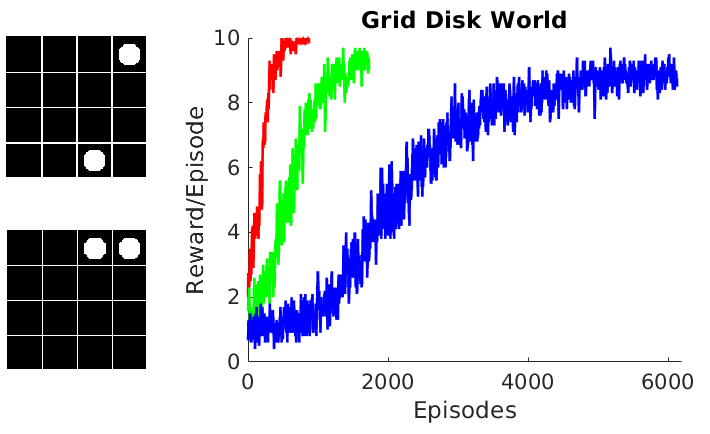}
    \caption{Left: the \textsc{grid-disk} domain. Disks are initially placed on a grid randomly (top left). The agent must pick one disk and place it adjacent to another (bottom left). Right: learning curves averaged over 10 runs for DQN applied to \textsc{grid-disk} for a $3 \times 3$ grid (red), a $4 \times 4$ grid (green), and a $5 \times 5$ grid (blue).}
    \label{fig:grid-disk}
\end{figure}

In order to use the MDP framework for move-effect systems, we need to define state and action sets. An action is defined to be a pair, $a = (a_m,a_e) \in \mathcal{A}$ where $\mathcal{A} = \mathcal{A}_M \times \mathcal{A}_E$. The agent does not observe state directly, but it can observe an image $I \in [0 \dots 2^{16}]^{|grid|} = \mathcal{I}$, taken at the beginning of each time step taken via a camera and/or depth sensor. Here, $grid \subset \mathbb{R}^d$ denotes the positions of a finite set of points corresponding to the pixels or voxels in the image or signed distance function. The system also has access to the configuration of its effectors, $\theta \in \Theta$, obtained using joint position sensors. We define state to be a history of the $k$ most recent observations and actions:
\begin{equation}
s_t = \langle I_t^{1-k}, \theta_t^{1-k}, a_t^{1-k}, \dots ,I_t^{-1}, \theta_t^{-1}, a_t^{-1} , I_t, \theta_t \rangle,
\label{eqn:state_definition}
\end{equation}
where $I_t$ and $\theta_t$ denote the current values for those two variables and $I_t^{-i}$, $\theta_t^{-i}$, and $a_t^{-i}$ denote the respective variables as they were $i$ time steps in the past (i.e. $i$ time steps prior to $t$). $\mathcal{S} = \mathcal{I}^{k} \times \Theta^{k} \times \mathcal{A}^{k-1}$ is the set of all states. History based representations of state as in Equation~\ref{eqn:state_definition} are often used in deep reinforcement learning (e.g. \cite{mnih2013playing}). This system is Markov for a sufficiently large value of $k$.


\subsection{Why move-effect problems are challenging}

\label{sect:move_effect_example}

The standard MDP formulation described in the previous section is not well suited to move-effect systems. For example, consider the \textsc{grid-disk} domain as shown in Figure~\ref{fig:grid-disk}, where the argent must pick one disk and place it horizontally adjacent to the other. Transitions are deterministic: a pick (resp. place) succeeds if executed for an occupied (resp. unoccupied) cell and does not otherwise. This example is a move-effect system where $\mathcal{A}_M \subset SE(2)$ is the set of 16 grid positions and $\mathcal{A}_E$ contains exactly one pick and one place action (32 actions total). On each time step, the agent observes an image of the grid as well as a single bit that denotes the configuration of its hand -- either open or closed (Equation~\ref{eqn:state_definition} for $k=1$). Using DQN on the MDP formulation described above, the number of episodes needed to learn a good policy increases as the number of cells in the grid increases, as shown in Figure~\ref{fig:grid-disk}, right. In this experiment, we used a standard $\epsilon$-greedy DQN with dueling networks~\cite{wang2015dueling}, no prioritized replay~\cite{schaul2015prioritized}, a buffer size of 10k, a batch size of 10, and an episode length of 10 steps. Epsilon decreased linearly from 100\% to 10\% over the training session. The neural network has two convolutional+relu+pooling layers of 16 and 32 units respectively with a stride of 3 followed by one fully connected layer with 48 units. We use the Adam optimizer with a learning rate of 0.0003.

\section{Deictic image mapping}

Move-effect systems have structure that makes problems involving these systems easier to solve than the results in Figure~\ref{fig:grid-disk} suggest. Notice that in the case of \textsc{grid-disk}, the optimal policy is most easily expressed relative to current disk positions on the grid: the agent must learn to pick up a disk and place it horizontally or vertically adjacent to the other. This reflects a symmetry in the problem whereby world configurations where objects occupy the same relative configurations have similar (i.e. symmetric) transitions and reward outcomes. In order to improve learning performance, we need to encode the problem in a way that reflects this symmetry. Note that this cannot be accomplished just by switching to an actor critic method like DDPG~\cite{lillicrap2015continuous}. DDPG could make it easier for the agent to learn to generalize over position, but it cannot generalize over orientation. Instead, we introduce deictic image state and action mappings that induce an abstract MDP that captures problem symmetries and can be solved using DQN.

\subsection{Action mapping}

Recall that in a move-effect system, each action is a move-effect pair,  $a_t = \langle a_m(t), a_e(t) \rangle \in \mathcal{A} = \mathcal{A}_M \times \mathcal{A}_E$, where we use the notation $a_m(t)$ to describe the destination of the base motion and $a_e(t)$ to describe the effector motion at time $t$. The hard part about expressing this action is expressing the $a_m$ component: whereas $\mathcal{A}_E$ is relatively small, $a_m$ could be any position and orientation within the work space of the manipulator. Our key idea is to express $a_m$ in terms of the appearance of the world in the vicinity of $a_m$ rather than as coordinates in a geometeric space. Recall that on each time step, the move-effect system observes an image or signed distance function, $I \in \mathcal{I} = [0 \dots 2^{16}]^{|grid|}$, expressed over a finite set of positions, $grid \subset \mathbb{R}^d$, corresponding to pixels or voxels. Let $crop(I,a_m) \in \mathcal{I}' = [0 \dots 2^{16}]^{|\overline{grid}|}$ denote a cropped region of $I$ centered on $a_m$ and aligned with the basis axes of $a_m$, defined over the positions $\overline{grid} \subset grid$. $\overline{grid}$ is assumed to correspond to a closed region of $I$. In our experiments, it is a square or cube centered on and aligned with $a_m$. The action mapping is then $g_s : \mathcal{A} \rightarrow \mathcal{A}' = \mathcal{I}' \times \mathcal{A}_E$ where:
\begin{equation}
\label{eqn:action_mapping}
g_{s_t}(a_t) = \langle crop(I_t,a_m(t)), a_e(t) \rangle,
\end{equation}
$I_t$ is an element of $s_t$, and $a_m(t)$ and $a_e(t)$ are elements of $a_t$. We call this a \textit{deictic image} action mapping because it uses an image to encode the motion relative to other objects in the world. This can be viewed as an action abstraction~\cite{ravindran2004algebraic,ravindran2003smdp} and we will call $\mathcal{A}'$ the \textit{abstract action} set.

\subsection{State mapping}

The action mapping of Equation~\ref{eqn:action_mapping} induces a state abstraction. Recall that we defined state to be the history of the last $k$ observations and actions (Equation~\ref{eqn:state_definition}). We can simplify this representation by using the action abstraction of Equation~\ref{eqn:action_mapping}. The easiest way to do this is to define abstract state to be the current robot state paired with a history of the $k-1$ most recently executed abstract actions:
\begin{eqnarray}
\nonumber
f_k(s_t) = \langle \{ & crop(I_t^{1-k},a_m^{1-k}(t)), a_e^{1-k}(t), \dots,  \\
& crop(I_t^{-1},a_m^{-1}(t)), a_e^{-1}(t), \theta_t \rangle
\label{eqn:state_abstraction_original_crops}
\end{eqnarray}
where $I_t^{-i}$, $a_m^{-i}(t)$, and $a_e^{-i}(t)$ are elements of $s_t$ and $\mathcal{S}' = \mathcal{I}'^{k-1} \times \mathcal{A}_E^{k-1} \times \Theta$ is the set of abstract states. When $k$ is understood, we will sometimes abbreviate $f = f_k$. We refer to $\mathcal{S}'$ as the \textit{abstract state set}.



\begin{figure}
    \centering
    \includegraphics[width=0.4\textwidth]{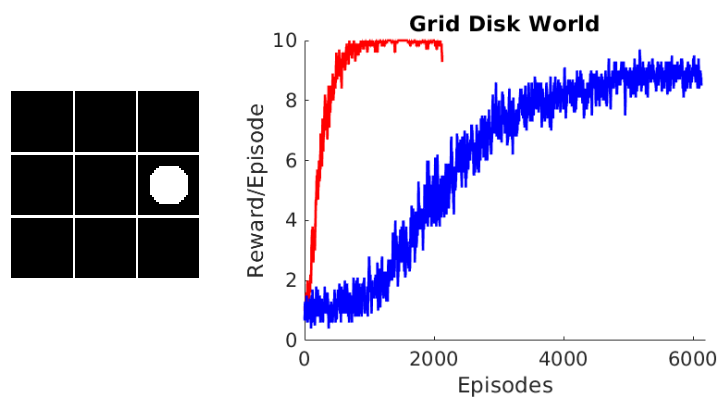}
    \caption{Left: example of a deictic action representation. Right: deictic (red) versus DQN baseline (blue) for $5 \times 5$ \textsc{grid-disk}.}
    \label{fig:deictic_v_baseline}
\end{figure}

\subsection{DQN in the abstract state and action space}

\label{sect:abstractdqn}

The state and action mappings introduced above induce an abstract MDP that can be solved using DQN. Let $Q' : \mathcal{S}' \times \mathcal{A}' \rightarrow \mathbb{R}$ denote the abstract action-value function, expressed over the abstract state and action space, encoded by a deep neural network. At the beginning of each time step, we evaluate the abstract state, $s_t' = f(s_t)$, for the current state, $s_t$. We can calculate the greedy action for the current state $s_t$ with respect to the abstract action-value function $Q'$ using:
\begin{equation}
a^* = \arg\max_{a \in \mathcal{A}} Q'(f(s_t),g_{s_t}(a)).
\end{equation}
After each time step, we store the underlying transition, $(s_t,a_t,s_{t+1},r_t)$ in the standard way. Training is nearly the same as usual except that after sampling a mini-batch, we calculate targets using the abstract state-action value function, $r_t + \max_{a \in \mathcal{A}} Q'(f(s_t),g_{s_t}(a))$, rather than in the standard way. The  neural network that encodes $Q'$ takes an abstract state-action pair as input and outputs a single estimate of value. In our experiments, we use a standard convolutional architecture comprised of two convolution+relu+pooling layers followed by fully connected layers.

Consider the following example of deictic image mapping for the \textsc{grid-disk} domain introduced earlier. Recall that for an image $I \in \mathcal{I}$ and a base motion $a_m \in \mathcal{A}_M$, $crop(I,a_m)$ is a cropped region of $I$ centered and aligned with $a_m$. $\mathcal{A}_M$ is the set of base motion target locations corresponding to the cells ($|\mathcal{A}_M| = 16$ for $4 \times 4$ \textsc{grid-disk}). In this example, we define $crop(I,a_m)$ to be the $3 \times 3$ cell square region centered on $a_m$. (We use zero-padding for $a_m$ positions on near the edge of the grid.) For example, Figure~\ref{fig:deictic_v_baseline} left shows $crop(I,a_m)$ for the place action shown in Figure~\ref{fig:grid-disk}b. Since the agent is to place a disk to the left of an existing disk, the deictic image for this place action shows the existing disk to the right of an empty square where the place will occur. $\theta$ is the configuration of the effector: in this case just the configuration of the gripper jaws. Figure~\ref{fig:deictic_v_baseline} right compares learning curves for DQN with deictic image states and actions versus the DQN baseline averaged over 10 runs each for the $5 \times 5$ \textsc{grid-disk} domain. DQN is parameterized just as it was in the last section. Notice that deictic image mapping speeds up learning considerably. 



\section{Scaling up}
\label{sect:scaling_up}

A key challenge for deictic image mapping is the large branching factor that is created by using end-to-end planned motions rather than small displacements as actions. In standard DQN, we must evaluate the Q-value of each action in order to select one for execution. However, many problems of practical interest involve tens or hundreds of thousands of potential actions. One approach to handling this would be to adapt an actor critic method such as DDPG to the deictic image mapping setting. However, we have so far only explored the DQN version of our approach and leave an actor critic version for future work. Instead, this section introduces several techniques that together enable us to handle the large branching factor for problems in SE(2). Note that large action spaces are becoming more common in the literature. For example, \cite{kunli_rss2018} uses a hand-coded heuristic to prune action choices and \cite{zeng2018learning} uses a fully convolutional network architectures to evaluate Q-values for millions of actions.


\noindent
\textbf{Passing multiple actions as a batch to the Q-network:} Perhaps the easiest way to handle the large number of actions is to pass a large set of action candidates to the Q-network as a single batch. Neural network back-ends such as TensorFlow are designed for this and it enables us to evaluate as many as 4.6k actions in a single forward pass on a standard NVIDIA 1080 GPU.

\noindent
\textbf{Keeping an estimate of the state-value function:} The large action branching factor is a problem every time it is necessary to evaluate $\max_{a \in \mathcal{A}} Q'(f(s_t),g_{s_t}(a))$. This happens twice in DQN: once during action selection and again during training when evaluating target values for a mini-batch. We eliminate the second evaluation by estimating the abstract state-value function, $V'$, concurrently with the abstract action-value function, $Q'$, using a separate neural network. $V'$ is trained using targets obtained during action selection: each time we select an action for a given state, we update the $V'$ network using the just-evaluated max for that state. This enables us to calculate a target using $r_t + \gamma V'(f(s_{t+1}))$ rather than $r_t + \gamma \max_{a \in \mathcal{A}} Q'(f(s_t),g_{s_t}(a))$. 

\noindent
\textbf{Value function hierarchy:} Another way to speed up evaluation of $\max_{a \in \mathcal{A}} Q'(f(s_t),g_{s_t}(a))$ is to leverage a hierarchy of value functions specifically designed for move-effect problems in $SE(2)$. Rather than estimating $Q'$ directly using a single neural network, we use a hierarchy of two networks, $Q_1'$ and $Q_2'$ (although one could envision additional levels of hierarchy). $Q_2'$ is trained in the standard way: for each $s,a$ pair in a mini-batch, $Q_2'(f(s),g_s(a))$ is trained with the corresponding target. However, we also train $Q_1'(f(s),g_s(Fix(a)))$ with the same target, where $Fix(a) \in SE(2)$ denotes the pose with the same position as $a$ but with an orientation that is fixed with respect to the base reference frame. Essentially, the $Q_1'$ estimate is an average of the value function over all orientations for each given position. 

Instead of maximizing by evaluating $Q_2'$ for the position and orientation of all possible motion actions, we score each position using $Q_1'$ and then maximize $Q_2'$ over the positions and orientations corresponding to the top scoring positions under $Q_1'$. This is essentially a graduated non-convexity method~\cite{blake1987visual}. First, we evaluate $Q_1'(f(s),g_s(\bar{a}))$ over all positions, $\bar{a} \in \bar{\mathcal{A}} = \{Fix(a) : a \in \mathcal{A}\}$. Let $\bar{\mathcal{A}}_{top}$ denote the top scoring $\eta$ percent of the abstract actions in $\bar{\mathcal{A}}$. Then, we evaluate $\max_{a \in \bar{\mathcal{A}}_{top}} Q_2'(f(s),g_s(Fix^{-1}(a)))$, where $Fix^{-1}$ denotes the one-to-many inverse of $Fix$ that returns all fully-specified poses that correspond to positions in $\bar{\mathcal{A}}_{top}$. This approach enables us to estimate the maximum of $Q_2'$ without exhaustively evaluating the function for all poses in $SE(2)$. Like all graduated non-convexity methods, this approach is not guaranteed to provide an optimum. However, our results indicate that it works reasonably well on our problems (see the comparison in the next section).



\noindent
\textbf{Using hand-coded heuristics:} Another way to reduce the action branching factor is to encode human knowledge about which kinds of motions are likely to be relevant to the problem. This type of approach has been used by others including \cite{kunli_rss2018} who hand-coded a heuristic for identifying which push actions. In this paper, we constrain the system to move only to positions that are within a neighborhood of some other visible object: we discard all motions, $a_m$, where the associated cropped image patch, $crop(I,a_m)$, does not include pixels with some positive height above the table plane.

\noindent
\textbf{Curriculum learning:} Another thing that can speed up training is curriculum learning~\cite{bengio2009curriculum}. This basically amounts to training the agent to solve a sequence of progressively more challenging tasks. The hard part is defining the task sequence in such a way that each task builds upon knowledge learned in the last. However, this is particularly easy with deictic image mapping. In this case, we just need to vary the discretization of $\mathcal{A}_M$ on each curriculum task, starting with a coarse and ending with a fine discretization. For example, suppose we want to learn a \textsc{grid-disk} policy over a fine grid of possible disk locations. We would initially train using a coarse discretization and subsequently train on a more fine discretization. Although some of the image patches in the fine discretization will be different from what was experienced in the coarse discretization, these new image patches will have a similar appearance to nearby patches in the coarse discretization. Essentially, base motions that move to similar locations will be representated by similar image patches. As a result, policy knowledge learned at a coarse level will generalize appropriately at the fine level.

\section{Experiments}

We evaluate deictic image mapping in simulation and on a robot for problems with an action set that spans $SE(2)$.

\begin{table*}[h!]
  \centering
  \begin{tabular}{| r | c | c | c | c | c | c | c | c |} 
    \hline
    Curriculum stage number & 1 & 2 & 3 & 4 & 5 & 6 & 7 & 8 \\
    \hline
    Object Type & Disks & Disks & Blocks & Blocks & Blocks & Blocks & Blocks & Blocks \\
    \hline
    Num Positions & 25 & 25 & 25 & 25 & 81 & 289 & 841 & 841 \\
    \hline
    Num Orientations & 2 & 8 & 2 & 4 & 8 & 8 & 8 & 16 \\
    \hline
    Num Actions & 100 & 400 & 100 & 200 & 1.2k & 4.6k & 13.5k & 26.9k \\
    \hline
  \end{tabular}
  \vspace{0.2cm}
  \caption{Eight-stage curriculum used in the simulated experiments.}
  \label{table:curriculum}
\end{table*}

\begin{figure}
    \centering
    \includegraphics[width=0.45\textwidth]{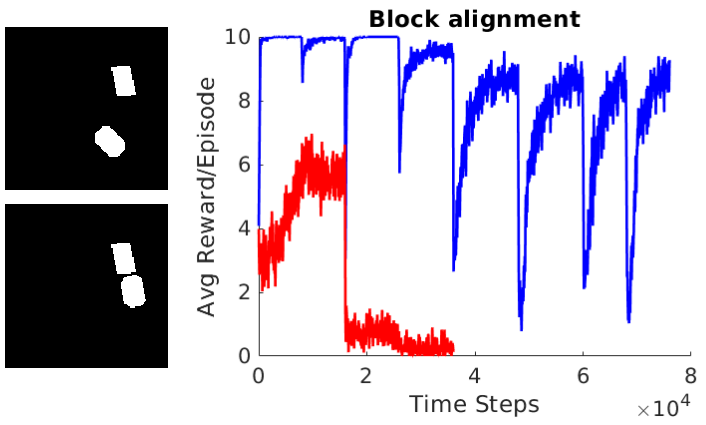}
    \caption{Left: block alignment problem. Top left: two blocks are placed randomly. Bottom left: the agent must pick up one block and place it roughly in alignment with the second. Right: learning curves for simulated blocks world domain. Blue: DIM with hierarchical value function only in last two curriculum steps. Red: DQN baseline.}
    \label{fig:blocks_world_results}
\end{figure}

\subsection{Block alignment in simulation}
\label{sect:sim_blocks_world}

We evaluate on the block alignment problem shown on the left side of Figure~\ref{fig:blocks_world_results} where the agent must grasp one block and place it in alignment with the other. We use a dense discretization of $SE(2)$ in this problem: the agent can choose to perform either a pick or place action (2 possibilities) at any position on a $29 \times 29$ grid (841 possibilities) at one of 16 orientations over 180 deg (equivalent to 32 orientations over 360 deg because of gripper symmetry), for a total of 26.9k different actions (last column of Table~\ref{table:curriculum}). We perform experiments to answer the following questions. (1) Can deictic image mapping solve the block alignment problem and how does it perform relative to a flat DQN solution? (2) How much does curriculum learning help? (3) What effect does the value function hierarchy have on learning?

\noindent
\textbf{(1) Comparison with a DQN baseline:} Here, we follow the eight-task curriculum shown in Table~\ref{table:curriculum}, starting with a task involving disks as in Figure~\ref{fig:grid-disk} and ending with the block alignment task. During the first six stages of curriculum learning, we use all of the scaling-up techniques described in the last section except for the value function hierarchy (as we show later, this slows learning performance). We only turn on the value function hierarchy in the last two stages because the NVIDIA 1080 ran out of memory without the hierarchy. Both DQN and deictic image mapping was parameterized exactly as it was earlier in the paper except that the deictic experiments used prioritized replay with $\alpha=0.6$, $\beta=0.4$, and $\epsilon=0.000001$ and we start $\epsilon$-greedy exploration with $\epsilon = 0.5$ instead of $\epsilon = 1$ (but $\epsilon$ still decreases linearly with time). The full training curriculum executes in approximately 1.5 hours on a standard Intel Core i7-4790K running one NVIDIA 1080 graphics card.

Figure~\ref{fig:blocks_world_results} shows the results. DQN performance is shown in red while deictic image mapping performance is shown in blue. Notice the ``spikes'' downward in the learning curve. Each spike corresponds to the start of a new task in the curriculum (a total of eight spikes including the one at episode zero). After the agent solves a task in the curriculum, this triggers a switch to a new task. Performance drops while the agent learns the new task but then recovers. The DQN agent was stopped after the fourth curriculum stage because average reward per episode had dropped to nearly zero. Notice that DQN completely fails to solve this task -- it only learns anything meaningful in the easiest stages of the curriculum. In contrast, the deictic agent does well throughout all eight curriculum stages. Note that average reward per episode is slightly lower on the fourth curriculum stage and onward because these tasks involve blocks in at least four orientations which can sometimes be infeasible. This occurs when both blocks are placed diagonally in a corner so that it is impossible to pick up either block and place it in alignment with the other.

\begin{figure}
    \centering
    \subfigure[]{\includegraphics[width=0.23\textwidth]{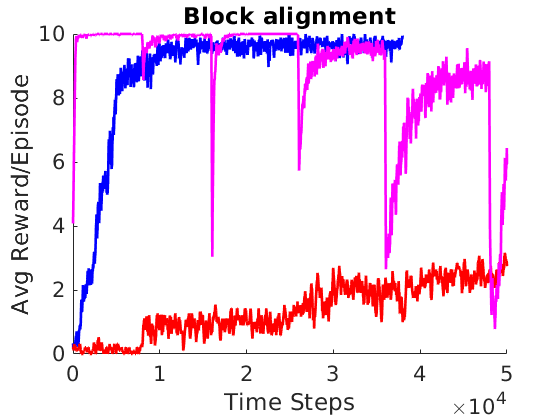}}
    \subfigure[]{\includegraphics[width=0.23\textwidth]{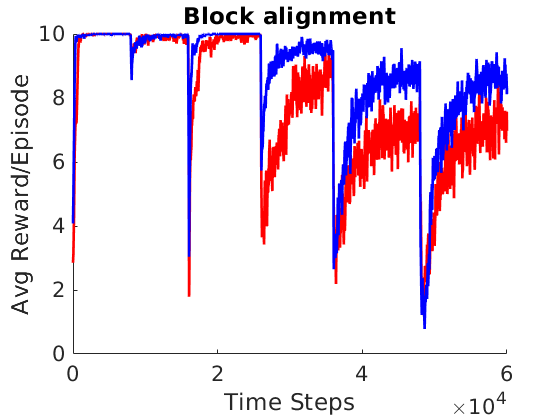}}
    \caption{(a) Ablation of curriculum learning. Blue, red: learning curves for curriculum stages 4 and 5, respectively, without running stages 1--3 first. Magenta: learning curve with curriculum. (b) Ablation of hierarchical value function. Blue: WITHOUT Q-function hierarchy; Red: WITH Q-function hierarchy.}
    \label{fig:blocks_world_results2}
\end{figure}

\noindent
\textbf{(2) The effect of curriculum learning:} We investigate learning performance if we were to start learning for stages 4 and 5 of the curriculum without first going through the prior curriculum steps. Here, the neural network is initialized with random weights and we allow the agent to learn for the same number of time steps as would be required by the curriculum to get to that point. Figure~\ref{fig:blocks_world_results2}a shows the results. The agent is able to learn stage 4 of the curriculum relatively quickly starting from random weights, but it cannot learn stage 5. This suggests that curriculum stages 1--3 are superfluous but that stages 5 and onward cannot be learned without some form of curriculum learning. Looking at Table~\ref{table:curriculum}, this makes sense because it is at stage 5 that the number of actions jumps from 400 to 1.2k.

\noindent
\textbf{(3) The effect of the value function hierarchy:} Another part of our strategy for speeding up the Q-function maximization operation is the value function hierarchy strategy outlined earlier. We compare learning performance for the first six curriculum stages with and without this additional feature. Recall that $\eta$ is the percentage of top scoring positions in $Q_1'$ that are selected to further evaluation in $Q_2'$. A larger $\eta$ results in faster evaluation, but a potentially worse estimate of the maximum. Here, we set $\eta=0.2$ Figure~\ref{fig:blocks_world_results2}b shows the comparison. Notice that the average reward per episode is approximately 15\% worse with the hierarchy as compared to without it. We conclude that while the value function hierarchy is sometimes needed to train in a reasonable period of time, it should be avoided if possible.

\subsection{Robot experiments}

We performed robot experiments to get a sense of the kinds of tasks that deictic image mapping can help solve. We used a UR5 equipped with a Robotiq two finger gripper, as shown in Figure~\ref{fig:real_robot}a. In all experiments, objects were dumped randomly onto a table in randomly selected positions within an $50 \times 50$ cm region within the manipulator workspace. All end-to-end motions commanded by the agent were planned using TrajOpt~\cite{schulman2013finding} with the final gripper pose constrained to be orthogonal to the table. All images were acquired using a depth sensor (Structure IO, the light blue device in Figure~\ref{fig:real_robot}a,b) mounted to the gripper. At the beginning of each time step, because of camera calibration inaccuracies, the robot took images from four different positions above the table and stitched them together to create a complete 2-D image. Then, the system selected and performed a pick or place action consisting of a single collision free motion followed by an opening or closing motion of the gripper. In all experiments, the system was trained in completely in simulation using OpenRAVE~\cite{Diankov2008} using a simulated setup nearly identical to the physical robot. Then, the policy learned in simulation was executed and evaluated on the real robot. We evaluated performance for the following five tasks. All tasks had sparse rewards: the agent received 0 reward everywhere except upon reaching the goal state.

\begin{figure*}
    \centering
    \subfigure[]{\includegraphics[height=0.123\textwidth]{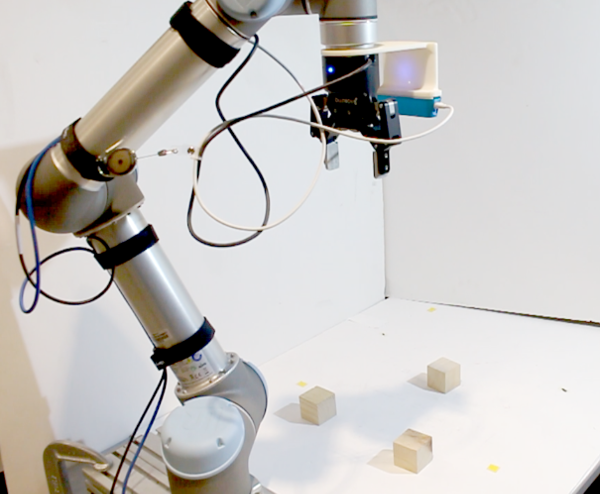}}
    \subfigure[]{\includegraphics[height=0.123\textwidth]{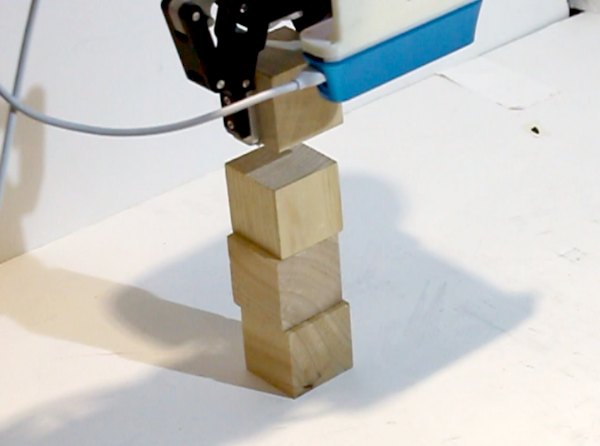}}
    \subfigure[]{\includegraphics[height=0.123\textwidth]{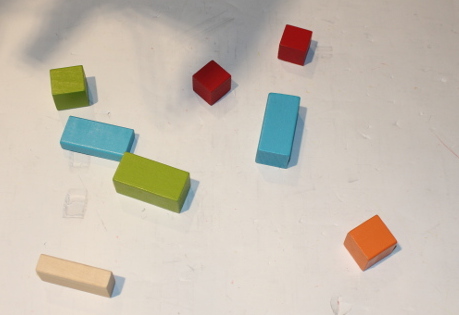}}
    \subfigure[]{\includegraphics[height=0.123\textwidth]{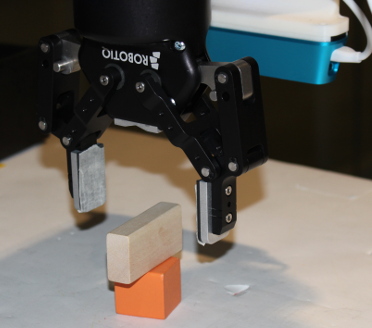}}
    \subfigure[]{\includegraphics[height=0.123\textwidth]{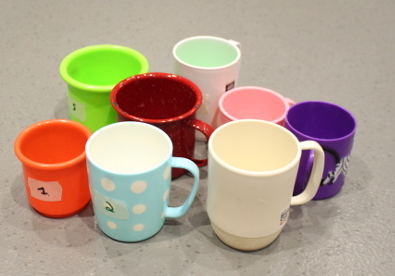}}
    \subfigure[]{\includegraphics[height=0.123\textwidth]{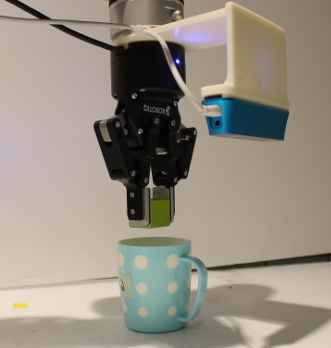}}
    \caption{(a) Experimental scenario. (b) four block stacking task. (c) initial block configuration for rectangular block stacking. (d) rectangular block stacking task. (d) Mugs used in cube-in-mug task. (f) Cube-in-mug task.}
    \label{fig:real_robot}
\end{figure*}

\begin{table}
  \centering
  \begin{tabular}{|c | c | c | c |} 
    \hline
    Task & Success  & Avg \# & Num \\
    & Rate & Steps & Trials \\
    \hline
    \hline
    3-in-row cube & 94\% & 7.8 & 50 \\
    \hline
    3 Cube Stack & 90\% & 4.7 & 50 \\
    \hline
    4 Cube Stack & 84\% & 6.9 & 50 \\
    \hline
    2 Rectangular Block Stack & 80\% & N/A & 83 \\
    \hline
    Cube in mug & 94\% & 2.72 & 50 \\
    \hline
  \end{tabular}
  \vspace{0.2cm}
  \caption{Results from robot experiments.}
  \label{table:robotresults}
\end{table}

\noindent
\textbf{Three-in-a-row cube arrangement:} In this task, three cubes were dumped onto a table and then re-oriented manually so that each block had a fixed orientation with respect to the robot. The system was trained using the same parameters as used in simulation, but using four curriculum stages starting with a 10cm step size and annealing down to a 1cm step size. The agent received +1 reward for achieving a three-in-a-row alignment within ten time steps and zero otherwise. As shown in the first row of Table~\ref{table:robotresults}, we obtain a 94\% task success rate over 50 trials.

\noindent
\textbf{Stacking three cubes:} This task is the same as three-in-a-row (above) except that the agent now received +1 reward for stacking the three cubes. We obtain a 90\% task success rate (second row of Table~\ref{table:robotresults}).

\noindent
\textbf{Stacking four cubes:} This task is the same as above except that the agent must now stack four cubes (Figure~\ref{fig:real_robot}b). We obtain an 84\% task success rate (third row of Table~\ref{table:robotresults}).

\noindent
\textbf{Stacking two rectangular blocks:} In this task, eight rectangular blocks were dumped on the table (Figure~\ref{fig:real_robot}c). The robot received +1 reward every time it stacked one block on to of another (Figure~\ref{fig:real_robot}d). After every successful stack, the top block was removed from the table by the experimenter. Unlike the three cube stacking/arranging tasks above, the action space for this task included 16 possible orientations (similar to the simulation experiments). The agent was trained using a seven stage curriculum and the value function hierarchy with $\eta=0.1$ that reduced the number of actions from 51.2k down to 8.32k. This task is interesting because the agent must learn to balance long blocks on top of small blocks as shown in Figure~\ref{fig:real_robot}d. The average success rate for two-block stacking attempts was 80\% (fourth row of Table~\ref{table:robotresults}).

\noindent
\textbf{Cube in mug:} Finally, we evaluated our approach on a task where the robot must pick up a cube and drop it into a mug (Figure~\ref{fig:real_robot}f). During robot testing, at the start of each episode, a random mug was chosen from the selection of eight mugs shown in Figure~\ref{fig:real_robot}e and was randomly placed on the table in a upright position along with a block. Importantly, these were novel mugs in the sense that they were different from any of the six randomly re-scaled mugs used to train the system. The agent received a reward of -0.1 when it moved the mug and +1 when it solved the task within no more than 10 time steps. The average task success rate here was 94\% (last row in Table~\ref{table:robotresults}).

These five robot experiments suggest that deictic image maps can be used to solve a variety of robotic manipulation problems related to pick and place. The stacking-four-cubes experiment demonstrates that the approach can solve problems that involve several pick/place actions. The stacking-two-rectangular blocks experiments shows that the method can perform relatively precise placements. Finally, the cube-in-mug experiment shows that the method can be trained to solve tasks involving novel objects.

\section{Analysis}
\label{sect:analysis}

It turns out that deictic image mapping is theoretically correct in the sense that optimal policies for the abstract MDP found using DQN induce optimal policies for the original system. We use the MDP homomorphism framework.


\begin{definition}[MDP Homomorphism, Ravindran 2004~\cite{ravindran2004algebraic}]
\label{defn:homomorphism}
An MDP homomorphism from an MDP $\mathcal{M} = \langle \mathcal{S}, \mathcal{A}, T, R \rangle$ onto $\mathcal{M}' = \langle \mathcal{S}', \mathcal{A}', T', R' \rangle$ is a tuple of mappings, $\langle f, \{ g_s | s \in \mathcal{S} \} \rangle$, where $f : \mathcal{S} \rightarrow \mathcal{S}'$ is a state mapping, $g_s : \mathcal{A} \rightarrow \mathcal{A}', s \in \mathcal{S}$ is a set of action mappings, and the following two conditions are satisfied for all $s \in \mathcal{S}$, $a \in \mathcal{A}$, and $s' \in \mathcal{S}'$:
\begin{eqnarray}
\label{eqn:homomorphism_condition1}
T'(f(s),g_s(a),s') & = & \sum_{\bar{s} \in \{ s \in \mathcal{S} | f(s)=s'\}} T(s,a,\bar{s}) \\
\label{eqn:homomorphism_condition2}
R'(f(s),g_s(a)) & = & R(s,a).
\end{eqnarray}
\end{definition}

$\mathcal{S}'$ and $\mathcal{A}'$ are the \textit{abstract} state and action sets, respectively. Recall that the action mapping introduced in Equation~\ref{eqn:action_mapping} encodes an action $a$ in terms of $crop(I,a_m)$. Since $I$ is the image or signed distance function perceived by sensors, this is a \textit{state-dependent} mapping. The MDP Homomorphism framework is specifically relevant to this situation because, unlike other abstraction frameworks~\cite{givan2003equivalence,dean1997model}, it allows for state dependent action abstraction. If we can identify conditions under which the action mapping of Equation~\ref{eqn:action_mapping} in conjunction with the state mapping of Equation~\ref{eqn:state_abstraction_original_crops} is an MDP homomorphism, then a theorem exists that says that a solution to the abstract MDP induces a solution in the original problem:
\begin{theorem}[Optimal value equivalence, Ravindran 2004~\cite{ravindran2004algebraic}]
\label{thm:homo_optimal}
Let $\mathcal{M}' = \langle \mathcal{S}', \mathcal{A}', T', R' \rangle$ be the homomorphic image of $\mathcal{M} = \langle \mathcal{S}, \mathcal{A}, T, R \rangle$ under the MDP homomorphism, $\langle f, \{ g_s | s \in \mathcal{S} \} \rangle$. For any $(s,a) \in \mathcal{S} \times \mathcal{A}$, $Q^*(s,a) = Q'^*(f(s),g_s(a)).$
\end{theorem}
Here, $Q^*(\cdot)$ and $Q'^*(\cdot)$ denote the optimal action-value function for the underlying MDP and the abstract MDP respectively. We now show that the deictic image mapping induces an MDP homomorphism.


\begin{theorem}
Given a move-effect system with state set $\mathcal{S}$; action set $\mathcal{A}$; transition function $T$ such that for all $t \in \mathbb{R}_{>0}$, $\theta_{t+1}$ is conditionally independent of $s_t$ given $crop(I_t,a_m(t))$; a deictic image mapping $\langle f, \{ g_s | s \in \mathcal{S} \} \rangle$; and an abstract reward function $R' : \mathcal{S}' \times \mathcal{A}' \rightarrow \mathbb{R}$; then there exist a reward function $R$ and an abstract transition function $T'$ for which $\langle f, \{ g_s | s \in \mathcal{S} \} \rangle$ is an MDP homomorphism from $\mathcal{M} = \langle \mathcal{S}, \mathcal{A}, T, R \rangle$ to $\mathcal{M}' = \langle \mathcal{S}', \mathcal{A}', T', R' \rangle$.
\label{thm:moveeffect_homo}
\end{theorem}

\begin{proof}
See Appendix 1 in supplementary materials.
\end{proof}

As a consequence of Theorem~\ref{thm:moveeffect_homo} and~\ref{thm:homo_optimal}, we know that optimal solutions to the abstract MDP induced by Equations~\ref{eqn:state_abstraction_original_crops} and~\ref{eqn:action_mapping} induce optimal solutions to the original move-effect system as long as the conditions of Theorem~\ref{thm:moveeffect_homo} are satisfied.

\section{Related Work, Limitations, and Discussion}


This work is related to prior applications of deixis in AI. Evidence exists that the human brain leverages deictic representations during performance of motor tasks~\cite{ballard1997deictic}. Whitehead and Ballard proposed an application of deixis to reinforcement learning, focusing on blocks world applications in particular~\cite{whitehead1991learning}. While our representations are only loosely related to Ballard's computational architecture, our work is closely related to the psychological evidence in \cite{ballard1997deictic}. 

The analysis section of this paper leverages the MDP Homomorphism framework introduced by Ravindran and Barto~\cite{ravindran2004algebraic,ravindran2003smdp}. MDP Homomorphisms are related to other MDP abstraction methods~\cite{givan2003equivalence,dean1997model}. However, unlike those methods, the MDP Homomorphisms allow for the state-dependent action abstraction, a critical part of our proposed abstraction.

This paper is also generally related to a variety of recent works exploring deep learning for robotic manipulation. The work of \cite{zeng2018learning} is related to ours in that they estimate a Q-function over a large number of grasping and pushing actions, in their case by creating a fully convolutional Q-network. The work of \cite{kunli_rss2018} also involves a large action space of push actions, pruned using a heuristic. Other pieces of recent related manipulation work are \cite{fang2018learning} and \cite{gualtieri_icra2018} who each propose methods of learning task-relevant grasps in the context of a reinforcement learning problem. Our work can be loosely viewed as an extension of recent grasp detection work by \cite{mahler2017dex} and \cite{gualtieri_iros2016}, however the focus here is on more complex manipulation tasks rather than just grasping.

In contrast to some of the work cited above, a nice aspect of our method is that it can be applied to a relatively general class of problems, such as those explored in our experiments. However, the method has important limitations. First, limits on the number of actions that can be realistically considered on a given time step means that it will be important to find ways to limit the number of motion actions that need to be considered. Another limitation is that the deictic image patch used to represent a motion only ``sees'' a portion of the total image. As a result, our method is not well suited to respond to non-local state information. Although we have not yet tried this, we expect that this problem can be solved by using a ``foveated'' action patch that ``sees'' the world around a potential action choice at multiple resolutions.

\bibliography{platt,main}  

\begin{thebibliography}{}

\bibitem[\protect\citeauthoryear{Ballard \bgroup et al\mbox.\egroup
  }{1997}]{ballard1997deictic}
Ballard, D.~H.; Hayhoe, M.~M.; Pook, P.~K.; and Rao, R.~P.
\newblock 1997.
\newblock Deictic codes for the embodiment of cognition.
\newblock {\em Behavioral and Brain Sciences} 20(4):723--742.

\bibitem[\protect\citeauthoryear{Bengio \bgroup et al\mbox.\egroup
  }{2009}]{bengio2009curriculum}
Bengio, Y.; Louradour, J.; Collobert, R.; and Weston, J.
\newblock 2009.
\newblock Curriculum learning.
\newblock In {\em Proceedings of the 26th annual international conference on
  machine learning},  41--48.
\newblock ACM.

\bibitem[\protect\citeauthoryear{Blake and Zisserman}{1987}]{blake1987visual}
Blake, A., and Zisserman, A.
\newblock 1987.
\newblock {\em Visual reconstruction}.
\newblock MIT press.

\bibitem[\protect\citeauthoryear{Dean, Givan, and Leach}{1997}]{dean1997model}
Dean, T.; Givan, R.; and Leach, S.
\newblock 1997.
\newblock Model reduction techniques for computing approximately optimal
  solutions for markov decision processes.
\newblock In {\em Proceedings of the Thirteenth conference on Uncertainty in
  artificial intelligence},  124--131.
\newblock Morgan Kaufmann Publishers Inc.

\bibitem[\protect\citeauthoryear{Diankov and Kuffner}{2008}]{Diankov2008}
Diankov, R., and Kuffner, J.
\newblock 2008.
\newblock Openrave: A planning architecture for autonomous robotics.
\newblock Technical Report CMU-RI-TR-08-34, Robotics Institute, Pittsburgh, PA.

\bibitem[\protect\citeauthoryear{Fang \bgroup et al\mbox.\egroup
  }{2018}]{fang2018learning}
Fang, K.; Zhu, Y.; Garg, A.; Kurenkov, A.; Mehta, V.; Fei-Fei, L.; and
  Savarese, S.
\newblock 2018.
\newblock Learning task-oriented grasping for tool manipulation from simulated
  self-supervision.
\newblock {\em arXiv preprint arXiv:1806.09266}.

\bibitem[\protect\citeauthoryear{Givan, Dean, and
  Greig}{2003}]{givan2003equivalence}
Givan, R.; Dean, T.; and Greig, M.
\newblock 2003.
\newblock Equivalence notions and model minimization in markov decision
  processes.
\newblock {\em Artificial Intelligence} 147(1-2):163--223.

\bibitem[\protect\citeauthoryear{Gualtieri \bgroup et al\mbox.\egroup
  }{2016}]{gualtieri_iros2016}
Gualtieri, M.; ten Pas, A.; Saenko, K.; and Platt, R.
\newblock 2016.
\newblock High precision grasp pose detection in dense clutter.
\newblock In {\em IEEE Int'l Conf. on Intelligent Robots and Systems}.

\bibitem[\protect\citeauthoryear{Gualtieri, ten Pas, and
  Platt}{2018}]{gualtieri_icra2018}
Gualtieri, M.; ten Pas, A.; and Platt, R.
\newblock 2018.
\newblock Pick and place without geometric object models.
\newblock In {\em IEEE Int'l Conf. on Robotics and Automation (ICRA)}.
\newblock IEEE.

\bibitem[\protect\citeauthoryear{Li, Hsu, and Lee}{2018}]{kunli_rss2018}
Li, J.~K.; Hsu, D.; and Lee, W.~S.
\newblock 2018.
\newblock Push-net: Deep planar pushing for objects with unknown physical
  properties.
\newblock In {\em Robotics: Science and Systems (RSS)}.

\bibitem[\protect\citeauthoryear{Lillicrap \bgroup et al\mbox.\egroup
  }{2015}]{lillicrap2015continuous}
Lillicrap, T.~P.; Hunt, J.~J.; Pritzel, A.; Heess, N.; Erez, T.; Tassa, Y.;
  Silver, D.; and Wierstra, D.
\newblock 2015.
\newblock Continuous control with deep reinforcement learning.
\newblock {\em arXiv preprint arXiv:1509.02971}.

\bibitem[\protect\citeauthoryear{Mahler \bgroup et al\mbox.\egroup
  }{2017}]{mahler2017dex}
Mahler, J.; Liang, J.; Niyaz, S.; Laskey, M.; Doan, R.; Liu, X.; Ojea, J.~A.;
  and Goldberg, K.
\newblock 2017.
\newblock Dex-net 2.0: Deep learning to plan robust grasps with synthetic point
  clouds and analytic grasp metrics.
\newblock {\em arXiv preprint arXiv:1703.09312}.

\bibitem[\protect\citeauthoryear{Mnih \bgroup et al\mbox.\egroup
  }{2013}]{mnih2013playing}
Mnih, V.; Kavukcuoglu, K.; Silver, D.; Graves, A.; Antonoglou, I.; Wierstra,
  D.; and Riedmiller, M.
\newblock 2013.
\newblock Playing atari with deep reinforcement learning.
\newblock {\em arXiv preprint arXiv:1312.5602}.

\bibitem[\protect\citeauthoryear{Ravindran and Barto}{2003}]{ravindran2003smdp}
Ravindran, B., and Barto, A.
\newblock 2003.
\newblock Smdp homomorphisms: An algebraic approach to abstraction in semi
  markov decision processes.
\newblock In {\em International Joint Conference on Artificial Intelligence
  (IJCAI)},  1011--1016.

\bibitem[\protect\citeauthoryear{Ravindran}{2004}]{ravindran2004algebraic}
Ravindran, B.
\newblock 2004.
\newblock {\em An algebraic approach to abstraction in reinforcement learning}.
\newblock Ph.D. Dissertation, University of Massachusetts at Amherst.

\bibitem[\protect\citeauthoryear{Schaul \bgroup et al\mbox.\egroup
  }{2015}]{schaul2015prioritized}
Schaul, T.; Quan, J.; Antonoglou, I.; and Silver, D.
\newblock 2015.
\newblock Prioritized experience replay.
\newblock {\em arXiv preprint arXiv:1511.05952}.

\bibitem[\protect\citeauthoryear{Schulman \bgroup et al\mbox.\egroup
  }{2013}]{schulman2013finding}
Schulman, J.; Ho, J.; Lee, A.; Awwal, I.; Bradlow, H.; and Abbeel, P.
\newblock 2013.
\newblock Finding locally optimal, collision-free trajectories with sequential
  convex optimization.
\newblock In {\em Robotics: Science and Systems}.
\newblock Citeseer.

\bibitem[\protect\citeauthoryear{Wang \bgroup et al\mbox.\egroup
  }{2015}]{wang2015dueling}
Wang, Z.; Schaul, T.; Hessel, M.; Van~Hasselt, H.; Lanctot, M.; and De~Freitas,
  N.
\newblock 2015.
\newblock Dueling network architectures for deep reinforcement learning.
\newblock {\em arXiv preprint arXiv:1511.06581}.

\bibitem[\protect\citeauthoryear{Whitehead and
  Ballard}{1991}]{whitehead1991learning}
Whitehead, S.~D., and Ballard, D.~H.
\newblock 1991.
\newblock Learning to perceive and act by trial and error.
\newblock {\em Machine Learning} 7(1):45--83.

\bibitem[\protect\citeauthoryear{Zeng \bgroup et al\mbox.\egroup
  }{2018}]{zeng2018learning}
Zeng, A.; Song, S.; Welker, S.; Lee, J.; Rodriguez, A.; and Funkhouser, T.
\newblock 2018.
\newblock Learning synergies between pushing and grasping with self-supervised
  deep reinforcement learning.
\newblock {\em arXiv preprint arXiv:1803.09956}.

\end{thebibliography}
\bibliographystyle{aaai}

\newpage

\section*{Appendix 1: Proof of Theorem~\ref{thm:moveeffect_homo}}

\begin{theorem}
Given a move-effect system with state set $\mathcal{S}$; action set $\mathcal{A}$; transition function $T$ such that for all $t \in \mathbb{R}_{>0}$, $\theta_{t+1}$ is conditionally independent of $s_t$ given $crop(I_t,a_m(t))$; a deictic image mapping $\langle f, \{ g_s | s \in \mathcal{S} \} \rangle$; and an abstract reward function $R' : \mathcal{S}' \times \mathcal{A}' \rightarrow \mathbb{R}$; then there exist a reward function $R$ and an abstract transition function $T'$ for which $\langle f, \{ g_s | s \in \mathcal{S} \} \rangle$ is an MDP homomorphism from $\mathcal{M} = \langle \mathcal{S}, \mathcal{A}, T, R \rangle$ to $\mathcal{M}' = \langle \mathcal{S}', \mathcal{A}', T', R' \rangle$.
\end{theorem}

\begin{proof}

The two conditions of Definition~\ref{defn:homomorphism} that must be satisfied are Equations~\ref{eqn:homomorphism_condition1} and \ref{eqn:homomorphism_condition2}. 

First, consider Equation~\ref{eqn:homomorphism_condition2}. Since we are given the abstract reward function, $R'$, we can define the underlying reward function to be $R(s,a) = R'(f(s),g_s(a)), \forall (s,a) \in \mathcal{S} \times \mathcal{A}$. Notice that the condition of the theorem that we are \textit{given} $R'$ means that it must be possible to express the reward function in the abstract state and action space.

Now, consider how to satisfy Equation~\ref{eqn:homomorphism_condition1}. We must identify an abstraction transition function $T'$ such that for all $s_t,a_t,s_{t+1}' \in \mathcal{S} \times \mathcal{A} \times \mathcal{S}'$,
\begin{equation}
\label{eqn:toshow}
P(s_{t+1}' | s_t,a_t) = T'(f(s_t),g_{s_t}(a_t),s_{t+1}')
\end{equation}
To show that $T'$ exists that satisfies Equation~\ref{eqn:toshow}, it is sufficient to show that $s_{t+1}'$ is conditionally independent of $s_t$ and $a_t$ given $f(s_t)$ and $g_{s_t}(a_t)$:
\begin{eqnarray}
\label{eqn:toshow2}
\nonumber
P(s_{t+1}' | f(s_t), g_{s_t}(a_t)) & = & P(s_{t+1}' | f(s_t), g_{s_t}(a_t), s_t, a_t) \\
& = & P(s_{t+1}' | s_t, a_t),
\end{eqnarray}
which is identical to Equation~\ref{eqn:toshow}. 

Notice, however, that since we can express $s_{t+1}'$ as
$s'_{t+1} = \langle f_{k-1}(s_t), g_{s_t}(a_t), \theta_{t+1} \rangle$, then $s_{t+1}'$ is conditionally independent of $s_t$ and $a_t$ given $f_{k-1}(s_t)$ (and therefore $f_k(s_t) = f(s_t)$), $g_{s_t}(a_t)$, and $\theta_{t+1}$. Since by assumption $\theta_{t+1}$ is conditionally independent of $s_t$ given $g_{s_t}(a_t)$ and we know that it is conditionally independent of $a_t$ given $a_e(t)$ by the definitions of those variables, we know that $s_{t+1}'$ is conditionally independent of $s_t$ and $a_t$ given $f(s_t)$ and $g_{s_t}(a_t)$. As a result, we have shown Equation~\ref{eqn:toshow2} and the theorem is proven.

\end{proof}

\end{document}